\documentclass{svproc}

\usepackage{multicol}
\usepackage{amsmath}
\usepackage{amsfonts}
\usepackage{amssymb}  
\usepackage{algorithm}
\usepackage{algpseudocode}
\usepackage{graphics} 
\usepackage{epsfig} 
\usepackage{wrapfig}
\usepackage[font={small}]{caption}
\usepackage{subcaption}

\usepackage{relsize}
\usepackage[numbers, sort, sectionbib]{natbib}
\usepackage[hyphens,spaces,obeyspaces]{url}

\begin{document}

\frontmatter          						
\pagestyle{headings}  						
\mainmatter              					
\title{Active Area Coverage from Equilibrium}
\author{Ian Abraham \and  Ahalya Prabhakar \and Todd D. Murphey
    \thanks{ \smaller
        This material is based upon work supported by the National Science Foundation under Grants CNS 1837515.
        Any opinions, findings and conclusions or recommendations expressed in this material are those of the authors and
        do not necessarily reflect the views of the aforementioned institutions.
        }
}
\authorrunning{Abraham et al.}				
\institute{Department of Mechanical Engineering, Northwestern University, \\2145 Sheridan Road, Evanston, IL 60208, USA\\
\email{i-abr@u.northwestern.edu}, \email{a-prabhakar@u.northwestern.edu}, \email{t-murphey@northwestern.edu}}
\maketitle              					

\vspace{-10mm}

\begin{abstract}
    This paper develops a method for robots to integrate stability into actively seeking out informative measurements through coverage.
    We derive a controller using hybrid systems theory that allows us to consider safe equilibrium policies
    during active data collection.
    We show that our method is able to maintain Lyapunov attractiveness while still actively seeking out data.
    Using incremental sparse Gaussian processes, we define distributions which allow a robot to actively seek out
    informative measurements.
    We illustrate our methods for shape estimation using a cart double pendulum, dynamic model learning of a hovering quadrotor,
    and generating galloping gaits starting from stationary equilibrium by learning a dynamics model for the half-cheetah
    system from the Roboschool environment.
\end{abstract}


\vspace{-6mm}

\section{Introduction}
    \vspace{-2mm}
    Robot learning has proven to be a challenge in real-world application.
    This is partially due to the ineffectiveness of passive data acquisition for learning and a necessity for action in
    order to generate informative data.
    What makes this problem even more challenging is that active data gathering is not a stable process.
    It involves exciting states in order to acquire new information.
    Safe exploration then becomes a challenge for modern day robotics.
    The problem becomes exacerbated when memory and task constraints (i.e., actively collecting data after deployment)
    are imposed on the robot.
    If the structures that compose the dynamics of the robot change over time, the robot will need to explore its own dynamics
    in a manner that is systematic and informative, avoiding damage to the underlying structures (and humans) in the environment.
    In this paper, we address these fundamental issues by developing an algorithm that is inspired by hybrid systems theory.
    This algorithm enables robots to actively pursue informative data by generating area coverage while guaranteeing
    Lyapunov attractiveness during exploration.

    Active data acquisition and learning are often considered part of the same problem of learning
    from experience~\cite{kormushev_robotmotorskills_em_rl, reinhart_AuRo_skill_babble}.
    This is generally seen in the field of reinforcement learning (RL) where attempts at a task, as well as learning from
    the outcome of actions, are used to both learn policies and predictive
    models~\cite{kormushev_robotmotorskills_em_rl, mckinnon_multimodal_gp_learning_online}.
    As a result, generalizing these methods to real-world application has been a topic of
    research~\cite{mckinnon_multimodal_gp_learning_online, kormushev_robotmotorskills_em_rl, tan_RSS_sim_to_real} where
    data-inefficiency dominates much of the progress.
    A solution to the problem of data-inefficiency is to simulate robots in a realistic virtual environment and subsequently
    use the large amount of synthetic data to solve a learning problem before applying the results on a real
    robot~\cite{marco2017virtual}.
    This leads to issues such as the ``reality-gap'' where finer modelling details such as motor delays lead to poor quality
    data for learning.

    Existing work addresses the data-inefficiency problem by actively seeking out informative data using
    information maximization~\cite{schwager_robotics_inf_gather} or by pruning a data-set based on some information
    measure~\cite{ nguyen_NEURO_incremental_sparse_gp}.
    These methods still suffer from the problem of local minima due to a lack of exploration or non-convex
    information objectives~\cite{ucinski_CRC_optimal_meas}.
    Safety in the task is also a concern when actively seeking out informative measurements.
    Methods typically provide some bound on the worst outcome model using probabilistic approaches~\cite{berkenkamp2017safe},
    but often only consider the safety with respect to the task and not with respect to the data collection process.
    We focus on problems where data collection involves exploring the state-space of robots where safe generation of
    informative data is important.
    In treating data acquisition as a dynamic area coverage problem\textemdash where the time spent during the trajectory of
    the robot is proportional to regions where there is an expectation of informative data\textemdash we are able to uncover
    more informative data that is not already expected.
    With this approach, we can provide attractiveness guarantees\textemdash that the robot will eventually return to a
    stable state\textemdash while providing control authority that allows the robot to actively seek out informative data in
    order to later solve a learning task.
    \textit{Thus, our contribution is an approach to dynamic area coverage for active data collection that starts
    from equilibrium policies for robots.}

    We structure the paper as follows: Section~\ref{sec-related-work} provides a list of related work,
    Section~\ref{sec-problem-statement} defines the problem statement for this work.
    Section~\ref{sec-algorithm} formulates the algorithm for active data acquisition from equilibrium.
    Section~\ref{sec-results} provides simulated and experimental examples of our method.
    Last, Section~\ref{sec-conclusion} provides concluding remarks on our method and future directions.

    \vspace{-4mm}

\section{Related Work} \label{sec-related-work}
    \vspace{-2mm}
    Existing work generally formulates problems of active data acquisition as information maximizing with respect to a
    known parameterized model~\cite{lin2017direct, bourgault2002information}.
    The problem with this approach is that robots need to address local
    optima~\cite{miller2016ergodic, bourgault2002information}, resulting in insufficient data collection.
    Other approaches have sought to solve this problem by thinking of information maximization as an area
    coverage problem~\cite{miller2016ergodic, ayvali2017ergodic}.
    Ergodic exploration, in particular, has remedied the issue of local extrema by using the ergodic metric to minimize the
    Sobelov distance~\cite{ arnold1992sobolev} from the time-averaged statistics of the robot's trajectory to the expected
    information in the explored region.
    This enables both exploration (quickly in low information regions) and exploitation (spending significant amount of time
    in highly informative regions) in order to avoid local extrema and collect informative measurements.
    The major downside is that this method assumes that the model of the robot is fully known.
    Moreover, there is little guarantee that the robot will not destabilize during the exploration process.
    This becomes an issue when the robot must explore part of its own state-space (i.e., velocity space) in order to
    generate informative data.
    To the authors' best knowledge this has not been done to this date.
    Another issue is that these methods do not scale well with the dimensionality of the search space, making experimental
    applications with this approach challenging due to computational limitations.

    Our approach overcomes these issues by using a sample-based KL-divergence measure~\cite{ayvali2017ergodic} as a replacement
    for the ergodic metric.
    This form of measure has been used previously; however, it relied on motion primitives in order to compute
    control actions~\cite{ayvali2017ergodic}.
    We avoid this issue by using hybrid systems theory in order to compute a controller that sufficiently reduces the
    KL-divergence measure from an equilibrium stable policy.
    As a result, we can use approximate models of dynamical systems instead of complete dynamic reconstructions in order to
    actively collect data while ensuring safety in the exploration process through a notion of attractiveness.

    The following section formulates the problem statement that our method solves.

    \vspace{-5mm}
\section{Problem Statement} \label{sec-problem-statement}

    \subsubsection*{Modeling Assumptions and Stable Policies}
        Assume we have a robot whose approximate dynamics can be modeled using
        \begin{equation}\label{eq:dynamics}
            \dot{x} = f(x,u) = g(x) + h(x) u
        \end{equation}
        where $x \in \mathbb{R}^n$ is the state of the robot,
        $u \in \mathbb{R}^m$ is a control vector applied to the robot,
        $g(x) : \mathbb{R}^n \to \mathbb{R}^n$ is the free unactuated dynamics,
        $h(x) : \mathbb{R}^n \to \mathbb{R}^{n \times m}$ is the actuated dynamics,
        and $\dot{x}$ is the time rate of change of the robot at state $x$ subject to the control $u$.
        Moreover, let us assume that there exists a Lyapunov function $V(x)$ such that under a
        policy $\mu(x) : \mathbb{R}^n \to \mathbb{R}^m$, $\dot{V}(x) < 0$ $\forall x \in \mathcal{B}$,
        where $\mathcal{B} = \{ x\in \mathbb{R}^n | \Vert x \Vert < r \}$ for $r > 0$.
        For the rest of the paper, we will refer to $\mu(x)$ as an equilibrium policy.

    \vspace{-4mm}
    \subsubsection*{KL-divergence and Area Coverage}
        Given the assumptions of known approximate dynamics and the equilibrium policy, we can define active exploration
        for informative data acquisition as automating safe switching between $\mu(x)$ and some control authority $\mu_\star(t)$
        that generates actions that actively seek out informative data.
        This is accomplished by specifying the active data acquisition task using an area coverage objective where we minimize
        the KL-divergence between the time average statistics of the robot along a trajectory and a spatial distribution
        defining the current coverage requirement.
        We can then define an approximation to the spatial statistics of the robot as follows:

        \begin{definition}
            Given a search domain $\mathcal{X}^v\subset\mathbb{R}^{n+m}$ where $v \le n + m$, the $\Sigma$-approximated
            time-averaged statistics of the robot, i.e., the time the robot spends in regions of the search domain
            $\mathcal{X}^v$, is defined by
            \begin{equation}\label{eq:time-averaged-stats}
                q(s \mid x(t), \mu(t)) = \frac{\eta}{T_r} \int_{t_i - t_\text{r}}^{t_i+T}
                    \exp \left[
                        -\frac{1}{2} \left(s - x_v(t) \right)^\top \Sigma^{-1} \left(s - x_v(t) \right)
                        \right]dt
            \end{equation}
            where $s \in \mathcal{X}^v \subset \mathbb{R}^{n+m}$ is a point in the search domain $\mathcal{X}^v$,  $x_v(t)$
            is the component of the robot's trajectory $x(t)$ and actions $\mu(t)$ that intersects the search domain
            $\mathcal{X}^v$,
            $\Sigma \in \mathbb{R}^{v \times v}$ is a positive definite matrix parameter that specifies the width of the Gaussian,
            $\eta$ is a normalization constant such that $q(s) > 0$ and $\int_{\mathcal{X}^v}q(s)ds = 1$,
            $t_i$ is the $i^\text{th}$ sampling time,
            and $T_r = T + t_r$ is sum of the time horizon $T$ and amount of time $t_r$ the robot remembers $x_v(t)$
            into the past.
        \end{definition}

        This is an approximation because the true time-averaged statistics, as described in~\cite{miller2016ergodic},
        is a collection of delta functions parameterized by time.
        We approximate the delta function as a Gaussian distribution with covariance $\Sigma$, converging as
        $\Vert \Sigma \Vert \to 0$.
        Using this approximation, we are able to relax the ergodic area-coverage objective in~\cite{miller2016ergodic} and use
        the following KL-divergence objective~\cite{ayvali2017ergodic}:
        \begin{equation*}
            \footnotesize
            D_\text{KL}(p \Vert q) = \int_{\mathcal{X}^v} p(s) \ln \frac{p(s)}{q(s)} ds = \mathbb{E}_{p(s)} \left[ \ln p(s) - \ln q(s) \right],
        \end{equation*}
        where $\mathbb{E}$ is the expectation operator,
        $q(s) = q(s \mid x(t), \mu(t))$,
        and $p(s)$, $p(s) >0, \int_{\mathcal{X}^v}p(s)ds=1$, is a distribution that describes where in the search domain an
        informative measurement is likely to be acquired.
        We can further approximate the KL-divergence via sampling where we approximate the expectation operator as
        \begin{equation}\label{eq:kl-objective}
            D_\text{KL} (p \Vert q) =  \mathbb{E}_{p(s)} \left[ \ln p(s) - \ln q(s) \right] \approx \sum_{i=1}^N p(s_i) \ln p(s_i) - p(s_i) \ln q(s_i),
        \end{equation}
        where $N$ is the number of samples in the search domain drawn from a uniform distribution.
        With this formulation, we can approximate the ergodic coverage metric using (\ref{eq:kl-objective}).

        In addition to the KL-divergence, we can add a task objective
        \begin{equation} \label{eq:task}
            J_\text{task} = \int_{t_i}^{t_i + T} \ell( x(t), \mu(x(t)) ) dt + m(x(t_i + T))
        \end{equation}
        where $t_i$ is the $i^\text{th}$ sampling time,
        $T$ is the time horizon,
        $\ell(x,u) : \mathbb{R}^n \times \mathbb{R}^m \to \mathbb{R}$ is the running cost,
        $m(x) : \mathbb{R}^n \to \mathbb{R}$ is the terminal cost,
        and $x(t) : \mathbb{R} \to \mathbb{R}^n$ is the state of the robot at time $t$.
        This additional objective will typically encode some other task, in addition to the KL-divergence objective.

        By summing the KL-divergence objective and a task objective (\ref{eq:task}), we can then pose active data acquisition
        as an optimal control problem subject to the initial approximate dynamic model of the robot.
        More formally, the objective is written as
        \begin{equation}\label{eq:objective}
            J = D_\text{KL}(p \Vert q) + \int_{t_i}^{t_i + T} \ell(x(t), \mu(x(t))) dt + m(x(t_i + T))
        \end{equation}
        where the goal is to generate a control $\mu_\star(t)$ that minimizes (\ref{eq:objective}) subject to the approximate
        dynamics (\ref{eq:dynamics}).
        Because we are including the equilibrium policy in the objective, we are able to synthesize controllers that take
        into account the equilibrium policy $\mu(x)$ and the desire to actively seek out measurements.

        For the rest of the paper, we assume the following:
        \begin{itemize}
        \item We have an initial approximate model $\dot{x} = f(x,u)$ of the robot.
        \item We also have an initial policy $\mu(x)$ that maintains the robot at equilibrium, for which there is a Lyapunov function.
        \end{itemize}
        These two assumptions are reasonable in that often robots are designed around stable states and typically
        have locally stable policies.

        The following section uses the fact that we have an initial policy $\mu(x)$ in order to synthesize control vectors
        $\mu_\star(t) : \mathbb{R} \to \mathbb{R}^m$ that reduce (\ref{eq:objective}).
        Specifically, we want to generate a hybrid composition of control actions that enable active data collection and
        actions that stabilize the robotic system.
        That way, it is possible to quantify how much the robot is deviating from a stable equilibrium.
        Thus, we motivate using hybrid systems theory in order to consider how much the objective (\ref{eq:objective})
        changes from switching from the equilibrium policy $\mu(x(t))$ to the control $\mu_\star(t)$.
        By quantifying the change, we specify an unconstrained optimization which solves for a control $\mu_\star$ that
        applies actions that retain Lyapunov attractiveness.

\vspace{-4mm}
\section{Algorithm} \label{sec-algorithm}
    \vspace{-2mm}
    Our algorithm starts by considering the objective defined in (\ref{eq:objective}) subject to the approximate dynamic
    constraints (\ref{eq:dynamics}) and policy $\mu(x)$.
    We want to quantify how sensitive the objective is to switching from policy $\mu(x(t))$ to the control vector
    $\mu_\star(t)$ for time $\tau \in \left[ t_i , t_i + T \right]$ for a infinitesimally small time duration $\lambda$.
    This sensitivity will be a function of $\mu_\star$ and inform us of the most influential time to apply $\mu_\star(t)$.
    Thus, we can use the sensitivity to write an objective whose minimizer is the schedule of control vectors $\mu_\star(t)$
    that reduces the objective (\ref{eq:objective}).
    \begin{proposition}
        The sensitivity of the objective (\ref{eq:objective}) with respect to the duration time $\lambda$, of switching from the
        policy $\mu(x)$ to the control $\mu_\star(t)$ at time $\tau$ is
        \begin{equation}\label{eq:mode-insertion}
            \frac{\partial J}{\partial \lambda} \Big |_{t= \tau} = \rho(\tau)^\top (f_2 - f_1)
        \end{equation}
        where $f_2 = f(x(t), \mu_\star(t))$ and $f_1 = f(x(t), \mu(x(t))$, and $\rho(t) \in \mathbb{R}^n$ is the adjoint,
        or co-state variable which is the solution of the following differential equation
        \begin{equation}
            \footnotesize
            \dot{\rho} =
                - \left(
                \frac{\partial \ell}{\partial x} + \frac{\partial \mu}{\partial x}^\top \frac{\partial \ell}{\partial u}
                -\frac{\eta}{T_r}\sum_i \frac{p(s_i)}{q(s_i)} \left(\frac{\partial g}{\partial x}
                + \frac{\partial\mu}{ \partial x}^\top\frac{\partial g}{\partial u}\right)
                \right) \\
                 - \left(
                \frac{\partial f}{\partial x} + \frac{\partial f}{\partial u} \frac{\partial \mu}{\partial x}
             \right)^\top \rho
        \end{equation}
        subject to the terminal constraint $\rho(t_i + T) = \frac{\partial}{\partial x} m(x(t_i + T))$.
    \end{proposition}
    \begin{proof}
        Taking the derivative of the objective (\ref{eq:objective}) with respect to the duration time $\lambda$ gives
        \[
            \frac{\partial }{\partial \lambda} J = \frac{\partial}{\partial \lambda} D_\text{KL}
            + \frac{\partial}{\partial \lambda} J_\text{task}.
        \]
        The term $\frac{\partial}{\partial \lambda} D_\text{KL}$ is calculated by
        \begin{align}\label{eq:kl-sensitivity}
            \footnotesize
            \frac{\partial}{\partial \lambda} D_{\text{KL}} \Bigg\vert_{t=\tau} & = -\sum_i  \frac{ p(s_i) }{q(s_i)} \frac{\eta}{T_r} \int_{\tau + \lambda}^{t_i + T}
            \left( \frac{\partial g}{ \partial x} + \frac{\partial\mu}{ \partial x}^\top\frac{\partial g}{\partial u} \right)^\top \frac{\partial x}{\partial \lambda}
            dt \nonumber
            \\
            & = -\sum_i  \frac{ p(s_i) }{q(s_i)} \frac{\eta}{T_r} \int_{\tau + \lambda}^{t_i + T}
            \left(\frac{\partial g}{ \partial x} + \frac{\partial\mu}{ \partial x}^\top\frac{\partial g}{\partial u} \right)^\top \Phi(t, \tau)
            dt \left(f_2 - f_1 \right)
        \end{align}
        where $g = g(s_i \mid  x(t), \mu(x(t)))= \exp\left[ -\frac{1}{2} \left( s_i - x_s(t) \right) \Sigma^{-1} \left( s_i - x_s(t) \right) \right]$,
        and $\Phi(t,\tau)$ is the state transition matrix for the integral equation
        \begin{equation}\label{eq:xdlambda}
            \frac{\partial x}{\partial \lambda} = (f_2 - f_1)
            + \int_{\tau + \lambda}^{t_i + T} \left( \frac{\partial f}{\partial x}
                + \frac{\partial f}{\partial u} \frac{\partial \mu}{\partial x}\right) ^\top
            \frac{\partial x}{\partial \lambda} dt
        \end{equation}
        where $f_2 = f(x(\tau), \mu_\star(\tau))$ and $f_1 = f(x(\tau), \mu(x(\tau))$.

        We can similarly show that the term $\frac{\partial}{\partial \lambda} J_\text{task}$ is given by
        \begin{align}\label{eq:task-sensitivity}
            \frac{\partial}{\partial \lambda} J_\text{task} \Bigg\vert_{t=\tau} & = \int_{\tau+\lambda}^{t_i + T} \left(
            \frac{\partial \ell}{\partial x} + \frac{\partial \mu}{\partial x}^\top \frac{\partial \ell}{\partial u}
            \right)^\top \frac{\partial x}{\partial \lambda} dt. \nonumber
            \\
            & = \int_{\tau+\lambda}^{t_i + T} \left(
            \frac{\partial \ell}{\partial x} + \frac{\partial \mu}{\partial x}^\top \frac{\partial \ell}{\partial u}
            \right)^\top \Phi(t, \tau) dt \left( f_2 - f_1 \right)
        \end{align}
        using the same expression in (\ref{eq:xdlambda}).
        Combining (\ref{eq:kl-sensitivity}) and (\ref{eq:task-sensitivity}) and taking the limit as $\lambda \to 0$ gives
            {\small
            \begin{equation} \label{eq:pre-lambda}
            \frac{\partial}{\partial \lambda} J = \int_{\tau}^{t_i + T}
            \left(
            \frac{\partial \ell}{\partial x} + \frac{\partial \mu}{\partial x}^\top \frac{\partial \ell}{\partial u}
             - \frac{\eta}{T_r}\sum_i \frac{p(s_i)}{q(s_i)} \left(\frac{\partial g}{ \partial x}
                + \frac{\partial\mu}{ \partial x}^\top\frac{\partial g}{\partial u} \right)
             \right)^\top \Phi(t, \tau) dt \left(f_2 - f_1 \right).
            \end{equation}}
        Setting
        \[
            \footnotesize
            \rho(\tau)^\top =   \int_{\tau}^{t_i + T}
            \left(
            \frac{\partial \ell}{\partial x} + \frac{\partial \mu}{\partial x}^\top \frac{\partial \ell}{\partial u}
             - \frac{\eta}{T_r}\sum_i \frac{p(s_i)}{q(s_i)} \left(\frac{\partial g}{ \partial x}
                + \frac{\partial\mu}{ \partial x}^\top\frac{\partial g}{\partial u} \right)
             \right)^\top \Phi(t, \tau) dt
        \]
        and from~\cite{axelsson_JOTA_modeinsertion} we can show that (\ref{eq:pre-lambda}) can be written as
        \[
            \frac{\partial}{\partial \lambda} J \Big |_{t=\tau} = \rho(\tau)^\top \left( f_2 - f_1 \right)
        \]
        where
        \[
            \footnotesize
            \dot{\rho} =
            - \left(
            \frac{\partial \ell}{\partial x} + \frac{\partial \mu}{\partial x}^\top \frac{\partial \ell}{\partial u}
            -\frac{\eta}{T_r}\sum_i \frac{p(s_i)}{q(s_i)} \left(\frac{\partial g}{ \partial x}
                + \frac{\partial\mu}{ \partial x}^\top\frac{\partial g}{\partial u} \right)
            \right) \\
             - \left(
            \frac{\partial f}{\partial x} + \frac{\partial f}{\partial u} \frac{\partial \mu}{\partial x}
             \right)^\top \rho.
        \]
        subject to the terminal condition $\rho(t_i + T) = \frac{\partial}{\partial x} m(x(t_i+T))$.
        \qed
    \end{proof}

    The sensitivity $\frac{\partial}{\partial \lambda}J$ is known as the mode insertion gradient~\cite{axelsson_JOTA_modeinsertion}.
    We can directly compute the mode insertion gradient for any control $\mu_\star$ that we choose.
    However, our goal is to find one such control $\mu_\star$ that reduces the objective (\ref{eq:objective}) but still
    maintains its value near the equilibrium policy $\mu(x)$.
    To solve for this control, we formulate the following objective function
    \begin{equation}\label{eq:secondary-obj}
        J_2 = \int_{t_i}^{t_i+T} \frac{\partial}{\partial \lambda} J \Big |_{t = \tau} +\frac{1}{2} \Vert \mu_\star(t) - \mu(x(t)) \Vert_R^2
    \end{equation}
    where $R\in \mathbb{R}^{m \times m}$ is a positive definite matrix that penalizes the deviation from the policy $\mu(x)$.
    \begin{proposition}
        The control vector that minimizes $J_2$ is given by
        \begin{equation} \label{eq:policy}
            \mu_\star(t) = - R^{-1} h(x(t))^\top \rho(t) + \mu(x(t)).
        \end{equation}
    \end{proposition}
    \begin{proof}
        Taking the derivative of (\ref{eq:secondary-obj}) with respect to $\mu_\star$ gives
        \begin{align}\label{eq:j2dmu}
            \frac{\partial}{\partial \mu_\star} J_2 & = \int_{t_i}^{t_i+T} \frac{\partial}{\partial \mu_\star} \left( \rho(t)^\top ( f_2 - f_1) \right)
            + R (\mu_\star (t) - \mu(x(t))) dt \nonumber \\
            & =  \int_{t_i}^{t_i+T}  h(x(t))^\top \rho(t) + R (\mu_\star (t) - \mu(x(t))) dt.
        \end{align}
        Since $J_2$ is convex in $\mu_\star$, we set the expression in (\ref{eq:j2dmu}) to zero and solve for $\mu_\star$ which gives us
        \begin{equation*}
            \mu_\star(t) = - R^{-1}h(x(t))^\top \rho(t) + \mu(x(t))
        \end{equation*}
        which is a schedule of control values that reduce the objective for time $t \in \left[t_i , t_i+T \right]$.
        \qed
    \end{proof}
    This controller reduces (\ref{eq:objective}) for $\lambda>0$ that is sufficiently small.
    The reduction in (\ref{eq:objective}), $\Delta J$, by applying $\mu_\star(\tau)$ can be approximated as
    $\Delta J \approx \frac{\partial}{\partial \lambda} J \lambda \mid_{t=\tau}$.
    Ensuring that $\frac{\partial }{\partial \lambda} J < 0$ is an indicator that the robot is always actively pursuing data and
    reducing the objective (\ref{eq:objective}).
    \begin{corollary}
        Let us assume that $\frac{\partial}{\partial \mu}\mathcal{H}\neq 0$ $\forall t \in \left[ t_i, t_i + T \right]$,
        where $\mathcal{H}$ is the control Hamiltonian.
        Then $\frac{\partial}{\partial \lambda} J < 0$ $\forall \mu_\star(t) \in \mathcal{U}$ where $\mathcal{U}$ is the
        control space.
    \end{corollary}
    \begin{proof}
        Inserting (\ref{eq:policy}) into (\ref{eq:mode-insertion}) gives
        \begin{align}\label{eq:neg-djdlam}
            \frac{\partial}{\partial \lambda} J &= \rho(t) ^\top \left( f_2 - f_1\right) \\
            & = \rho(t)^\top\left(
            g(x(t)) + h(x(t)) \mu_\star(t) - g(x(t)) - h(x(t)) \mu(x(t))
            \right) \nonumber.
        \end{align}
        Because of the manner in which we chose to solve for $\mu_\star(t)$, $g(x)$ and $\mu(x(t))$ cancel out in (\ref{eq:neg-djdlam}).
        In addition,  $\frac{\partial}{\partial \mu} \mathcal{H} \neq 0$ implies that $\frac{\partial}{\partial \lambda} J \neq 0$
        and the policy $\mu(x(t))$ is not an optimizer of (\ref{eq:objective}).
        As a result, we can further analyze $\frac{\partial}{\partial \lambda} J$ without the need to consider the policy $\mu(x)$.
        This gives us the following expression
        \begin{equation*}
            \frac{\partial}{\partial \lambda} J = - \rho(t)^\top h(x(t)) R^{-1} h(x(t))^\top \rho(t)
        \end{equation*}
        which we can rewrite as
        \begin{equation} \label{eq:neg-djdlam2}
            \frac{\partial}{\partial \lambda} J= - \Vert h(x(t))^\top \rho \Vert_{R^{-1}}^2 < 0.
        \end{equation}
        Thus, (\ref{eq:neg-djdlam2}) shows us that $\frac{\partial}{\partial \lambda}J$ is always negative subject to the
        schedule of control vectors (\ref{eq:policy}) and the objective is being reduced when $\mu_\star(t)$ is applied.\qed
    \end{proof}

    We automate the switching between $\mu(x(t))$ and $\mu_\star(t)$ by choosing a $\tau$ and $\lambda$ such that
    $\frac{\partial}{\partial \lambda} J$ is most negative and $\Delta J < 0$.
    This is done through the combination of choosing $\tau$ with a 1-dimensional optimization and solving for $\lambda$ using a
    line search until $\Delta J < 0$~\cite{mavrommati2018real, abrahamDecentralized2018}.
    By choosing $\lambda < T$ we can place a bound on how much our algorithm excites the dynamical system through
    Lyapunov analysis (Theorem~\ref{tmh1}).

    \begin{theorem}\label{tmh1}
        Assume there exists a Lyapunov function $V(x)$ for (\ref{eq:dynamics}) such that under the policy $\mu(x)$, $x(t)$ is
        asymptotically stable.
        That is, $\dot{V}(x) < 0$ $\forall \mu(x),x\in \mathcal{B}$ where
        $\mathcal{B} = \{ x\in \mathbb{R}^n | \Vert x \Vert < r \}$ for $r > 0$.
        Then, given the schedule of control vectors (\ref{eq:policy}) $\mu_\star(t)$
        $\forall t \in \left[ \tau, \tau + \lambda \right]$,
        $V(x(t)) - V(x(t), \mu(x(t))) \le \lambda \beta$, where
        $V(x(t), \mu(x(t)))$ is the Lyapunov function subject to the policy $\mu(x)$,
        and $\beta = \sup_{t \in \left[ \tau, \tau + \lambda \right]} - \frac{\partial V} {\partial x}h(x(t))R^{-1}h(x(t))^\top \rho(t)$.
    \end{theorem}
    \begin{proof}
        Writing the integral form of the Lyapunov function switching between $\mu(x(t))$ and $\mu_\star(t)$ at time $\tau$ for a duration of time $\lambda$ starting at $x(0)$ gives
            { \footnotesize
            \begin{align}\label{eq-lyap-switch}
                V(x(t)) & = V(x(0)) && + \int_{0}^{t} \dot{V}(x(s), \mu(x(s)) ) ds \nonumber \\
                & = V(x(0)) && + \int_{0}^{\tau}  \dot{V}(x(s), \mu(x(s)) ) ds  \\
                & && +  \int_{\tau}^{\tau+\lambda}  \dot{V}(x(s), \mu_\star(s) ) ds
                + \int_{\tau+\lambda}^{t}  \dot{V}(x(s), \mu(x(s)) ) ds, \nonumber
            \end{align}}
        where we explicitly write the dependency on $\mu(x(t))$ in $\dot{V}$.
        Using chain rule, we can write
        \begin{equation}\label{eq-lyap-chain}
            \dot{V}(x,u) = \frac{\partial V}{\partial x} f(x,u) = \frac{\partial V}{\partial x}g(x)
            + \frac{\partial V}{\partial x}h(x) u.
        \end{equation}
        By inserting (\ref{eq:policy}) in (\ref{eq-lyap-chain}) we can show the following identity:
        \begin{align}\label{eq-affine-prop}
            \dot{V}(x, \mu_\star) &= \frac{\partial V}{\partial x}g(x) + \frac{\partial V}{\partial x}h(x) \mu_\star\nonumber \\
            & = \frac{\partial V}{\partial x}g(x) + \frac{\partial V}{\partial x}h(x) \left(
            -R^{-1}h(x)^\top \rho + \mu(x)
            \right) \nonumber\\
            & = \dot{V}(x, \mu(x)) - \frac{\partial V}{\partial x} h(x) R^{-1} h(x)^\top \rho .
        \end{align}
        Using (\ref{eq-affine-prop}) in (\ref{eq-lyap-switch}), we can show that
        \begin{align}\label{eq-lyap-cool-form}
            V(x(t)) &= V(x(0)) + \int_{0}^{t} \dot{V}(x(s), \mu(x(s)) ) ds - \int_{\tau}^{\tau + \lambda}
                \frac{\partial V}{\partial x}h(x(s))R^{-1} h(x(s))^\top \rho(s) ds \nonumber \\
             &= V(x(t), \mu(x(t))) - \int_{\tau}^{\tau + \lambda} \frac{\partial V}{\partial x}h(x(s))R^{-1} h(x(s))^\top \rho(s) ds
        \end{align}
        where $V(x(t), \mu(x(t))) = V(x(0)) + \int_{0}^{t} \dot{V}(x(s), \mu(x(s)) ) ds$.

        Letting the largest value of $\frac{\partial V}{\partial x}h(x(s))R^{-1} h(x(s))^\top \rho(s)$ be given by
        $\beta = \sup_{t \in \left[ \tau, \tau + \lambda \right]} - \frac{\partial V} {\partial x}h(x(t))R^{-1}h(x(t))^\top \rho(t)>0,$
        we can approximate (\ref{eq-lyap-cool-form}) as
        \begin{align}
            V(x(t)) &= V(x(t), \mu(x(t)) - \int_{\tau}^{\tau + \lambda} \frac{\partial V}{\partial x}h(x(s))R^{-1} h(x(s))^\top \rho(s) ds\\
            & \le V(x(t), \mu(x(t))) + \beta \lambda.
        \end{align}
        Subtracting both side by $V(x(t), \mu(x(t)))$ gives the upper bound on instability
        \begin{equation}
        V(x(t)) - V(x(t), \mu(x(t))) \le \beta \lambda
        \end{equation} for the active data collection process.
        \qed
    \end{proof}
    By fixing the maximum value of $\lambda$, we can provide an upper bound to the change of the Lyapunov function during
    active data acquisition.
    Moreover, we can tune our control vector $\mu_\star(t)$ using the regularization value $R$ such that as
    $\Vert R \Vert \to \infty$, $\beta \to 0$ and $\mu_\star(t) \to \mu(x(t))$.
    With this bound, we can  guarantee Lyapunov attractiveness~\cite{polyakov2014stability}, where the system (\ref{eq:dynamics})
    is not Lyapunov stable, but rather there exists a time $t$ such that the system (\ref{eq:dynamics}) is guaranteed to
    return to a region of attraction where the system can be guided towards a stable equilibrium state $x_0$.
    This property will play an important role in examples in Section~\ref{sec-results}.
    \begin{definition}
        A dynamical system (\ref{eq:dynamics}) is Lyapunov attractive if at some time $t$, the trajectory of the system
        $x(t) \in \mathcal{C}(t) \subset \mathcal{B}$ where $\mathcal{C}(t) = \{ x(t) \in \mathbb{R}^n | V(x) \le c, \dot{V}(x(t)) <0\}$
        and $\lim_{t\to\infty}  x(t)  \to x_0$ such that $x_0$ is an equilibrium state.
    \end{definition}
    \begin{theorem}\label{thm2}
        Given the schedule of control vectors
        (\ref{eq:policy}) $\mu_\star(t)$ $\forall t \in \left[\tau, \tau + \lambda \right]$, the robotic system governed by the
        dynamics in (\ref{eq:dynamics}) is Lyapunov attractive such that $\lim_{t\to\infty}  x(t, \tau, \lambda)  \to x_0$, where
        {\smaller
        \begin{multline*}
            x(t, \tau, \lambda) = x(0)
            + \int_{0}^{\tau} f(x(s), \mu(x(s)) ds + \int_{\tau}^{\tau+\lambda} f(x(s), \mu_\star(s) ds
            + \int_{\tau+\lambda}^{t} f(x(s), \mu(x(s)) ds,
        \end{multline*}}
        is the solution to switching between stable and exploratory motions for duration $\lambda$ starting at time $\tau$.
    \end{theorem}
    \begin{proof}
        Assume there exists a Lyapunov function such that $\dot{V}(x) <0$ under the policy $\mu(x)$.
        Moreover, assume that subject to the control vector $\mu_\star(t)$, the trajectory
        $x(\tau+\lambda) \in \mathcal{C}(\tau + \lambda)\subset \mathcal{B}$ where
        $\mathcal{C}(t) = \{x(t) \in \mathbb{R}^n | V(x) \le c, \dot{V}(x(t), \mu(x(t)) < 0 \}$ where $c>0$.
        Using Theorem~\ref{tmh1}, the integral form of the Lyapunov function (\ref{eq-lyap-switch}), and the
        identity (\ref{eq-affine-prop}), we can write
        \begin{multline}
            V(x(t)) = V(x(0)) + \int_{0}^{t} \dot{V}(x(s), \mu(x(s)) ) ds\\
            - \int_{\tau}^{\tau + \lambda} \frac{\partial V}{\partial x}h(x(s))R^{-1} h(x(s))^\top \rho(s) ds
             \le V(x(0)) - \gamma t + \beta \lambda,
        \end{multline}
        where $ -\gamma = \sup_{s \in \left[0, t \right]} \dot{V}(x(s), \mu(x(s))) < 0$.
        Since $\lambda$ is fixed and $\beta$ can be tuned by the matrix weight $R$, we can choose a $t$ such that
        $\gamma t \gg \beta \lambda$.
        Thus, $\lim_{t\to \infty} V(x(t)) \to V(x_0)$ and $\lim_{t\to\infty}  x(t, \tau, \lambda)  \to x_0$, implying
        Lyapunov attractiveness,  where $V(x_0)$ is the minimum of the Lyapunov function at the equilibrium state $x_0$.
        \qed
    \end{proof}
    Asymptotic attractiveness shows us that the robot will return to a region where $V(x)$ will return to a minimum under
    policy $\mu(x)$, allowing the robot to actively explore and collect data safely.
    Moreover, we can choose the value of $\lambda$ and $\mu_\star$ in automating the active data acquisition such that
    attractiveness always holds, giving us an algorithm that is safe for active data collection.

    All that is left is to define a spatial distribution that actively selects which measurements are more informative to the learning task.

    \vspace{-3mm}
    \subsubsection*{Measure of Data Importance for Model Learning}
        Our goal is to provide a method that is general to any form of learning that requires a robot to actively seek out
        measurements through action.
        This may include area mapping or learning the dynamics of the robot.
        Thus, we use measures that allow the robot to quantify where in the search domain there exists useful data that
        needs to be collected.
        While there exists many measures that can select important data subject to a learning task, we use a measure of
        linear independence~\cite{scholkopf1999input, nguyen_NEURO_incremental_sparse_gp, yan_incremental_sparse_gp}.
        This measure is often used in sparse Gaussian processes~\cite{nguyen_NEURO_incremental_sparse_gp,yan_incremental_sparse_gp}
        where a data set $\mathcal{D}=\{x_i, y_i\}_{i=1}^M$ is comprised of $M$ input measurements $x_i\in \mathbb{R}^v$ and $M$
        output measurements $y_i \in \mathbb{R}^c$ such that each data point maximizes the measure of linear independence.
        We use this measure of independence, also known as a measure of importance, to create a distribution for which the robot
        will provide area coverage in the search domain for active data collection.

        As illustrated in ~\cite{nguyen_NEURO_incremental_sparse_gp}, this is done by evaluating a new measurement
        $x_{M+1}, y_{M+1}$ against the existing data points in $\mathcal{D}$ given the structure of the model that is being learned.
        \begin{definition}\label{eq-importance-measure}
            The importance measure $\delta \in \mathbb{R}^+$ for a new measurement pair $\{x_{M+1}, y_{M+1} \}$ is given by
            \begin{equation}
            \delta = k(x_{m+1}, x_{m+1}) - \bold{k}^\top \bold{a}
            \end{equation}
            which is the solution to $\delta = \Vert \sum_{i=1}^M a_i \phi(x_i) - \phi(x_{M+1}) \Vert ^2$,
            where $\phi(x)$ are the basis functions (also known as feature vectors)
            \footnote{This feature vector can be anything from a Fourier set of basis functions or a neural network.
            In addition, we can parameterize the functions $\phi(x)=\phi(x,\theta)$ and have the functions change over time.},
            $a_i$ is the coefficient of linear dependence,
            the matrix $K \in \mathbb{R}^{M \times M}$ is known as the kernel matrix with elements $K_{i,j} = k(x_i, x_j)$
            such that $k : \mathbb{R}^{v \times v} \to \mathbb{R}$ is the kernel function given by the inner product
            $k(x_i, x_j) = \langle \phi(x_i), \phi(x_j) \rangle$, $\bold{k} = [ k(x_1, x_{m+1}), k(x_2, x_{m+1}), \ldots, $ $k(x_m, x_{m+1})]^\top$,
            and $\bold{a} = K^{-1} \bold{k}$.
        \end{definition}

        The value $\delta$ provides a measure of how well the point $x_{M+1}$ can be represented given the existing data set
        and structure of the model being learned.
        Note that this measure will be computationally intractable for very large $M$.
        Instead, other measures like the expected information density derived from the Fisher information
        matrix~\cite{miller2016ergodic, emery_optimal_exp_design} can be used if the learning task has a model that is
        parameterized by a set of parameters $\theta$.
        Since $\delta > 0$, we define an importance distribution for which the robot will use generate area coverage.
        \begin{definition}
            The importance distribution is
            \begin{equation}\label{eq-importance-dist}
                p(s) = \frac{1}{\eta}\left( k(s, s) - \bold{k}(s)^\top \bold{a}(s) \right)
            \end{equation}
            where $\eta = \int_{\mathcal{X}^v} k(s,s) - \bold{k}(s)^\top \bold{a}(s)ds$, and $\bold{k}$, $\bold{a}$ are
            functions of points $s \in \mathcal{X}^v$.
        \end{definition}
        Note that $p(s)$ will change as $\mathcal{D}$ is augmented or pruned.
        If at any time $\delta > \delta_i$ for $i = \left[1, \ldots, M \right]$, we remove the $i^\text{th}$ point with the
        lowest $\delta$ value and add in the new data point.

        We provide an outline of our method in Algorithm~\ref{alg-AOL} for online data acquisition.
        The following section evaluates our method on various simulated environments.

        \begin{algorithm}[!h]
            \small
        \caption{Active Data Acquisition from Equilibrium }
        \label{alg-AOL}
        \centering
        \begin{algorithmic}[1]
        \State \textbf{initialize:} local dynamics model, initial condition $x(t_0), $ initial equilibrium policy $\mu(x)$,
            learning task model structure $\phi(x)$.
        \For{$i = 0, \ldots, \infty$}
            \State simulate $x(t)$ with $\mu(x(t))$ from $x(t_i)$ using dynamics model $f(x,u)$
            \State calculate $\rho(t)$ and $\frac{\partial}{\partial \lambda} J $
            \State compute control $\mu_\star(t)$ for $t \in \left[t_i, t_i+T\right]$
            \State choose $\tau, \lambda$ that minimizes $\frac{\partial}{\partial \lambda} J $
            \State apply $\mu_\star (\tau) \text{ if } t \in \left[ \tau, \tau + \lambda \right] \text{ else apply } \mu(x(t))$ to robot
            \State sample state $x(t_{i+1})$ and measurement $y(t_{i+1})$
            \State verify importance $\delta$ and update $p(s)$ if $\delta > \delta_i \forall i \in \left[1, \ldots, M \right]$
        \EndFor
        \end{algorithmic}
        \end{algorithm}

\vspace{-5mm}
\section{Simulated Examples}\label{sec-results}
    \begin{figure*}
        \centering
        \begin{subfigure}[b]{.35\textwidth}
            \centering
            \includegraphics[width=0.3\linewidth]{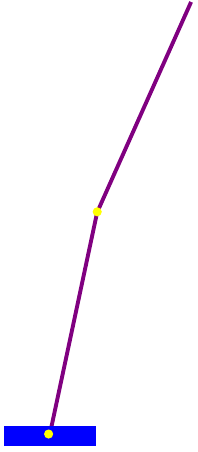}
            \caption{Cart Double Pendulum}
            \label{fig-env-1}
        \end{subfigure}
        \begin{subfigure}[b]{.35\textwidth}
            \centering
            \includegraphics[width=0.7\linewidth]{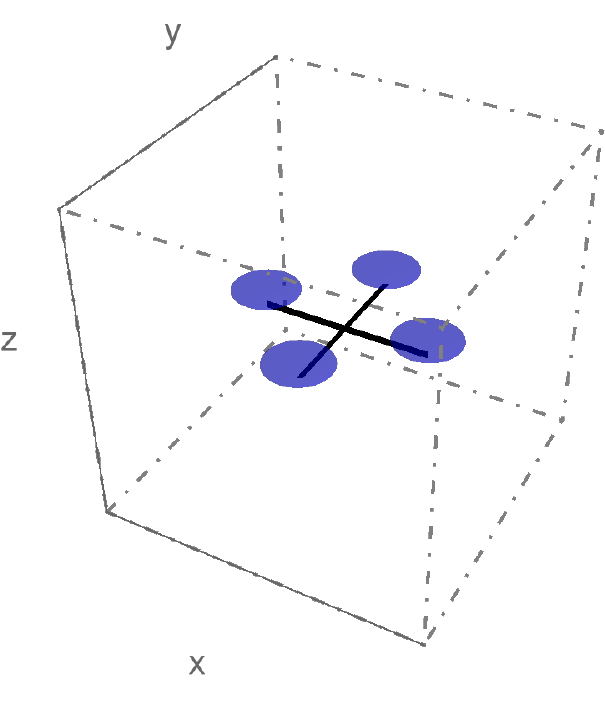}
            \caption{Quadrotor}
            \label{fig-env-2}
        \end{subfigure}
        \begin{subfigure}[b]{.28\textwidth}
            \includegraphics[width=\linewidth]{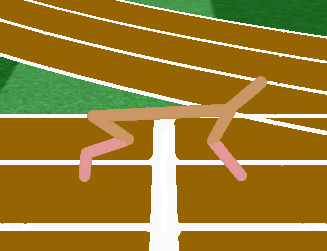}
            \caption{Half-Cheetah}
            \label{fig-env-3}
        \end{subfigure}
        \caption{
                Simulated experimental environments (a) cart double pendulum and (b) quadcopter, and (c) half-cheetah.
                The results for each system may be seen in the accompanying multimedia supplement.
        }
        \label{fig-envs}
        \vspace{-5mm}
    \end{figure*}

    In this section, we illustrate examples of Algorithm~\ref{alg-AOL} for different examples that may be encountered in robotics.
    Figure~\ref{fig-envs} depicts three robotic systems on which we base our examples.
    In the first example, we use a cart double pendulum for use in area coverage for shape estimation.
    In the second and third example, we use a 22 dimensional quadrotor~\cite{Fan-RSS-16} and a 26 dimensional half-cheetah model
    from Roboschool~\cite{klimov2017roboschool} for learning a dynamics model of the robotic systems by exploring in the state-space.
    For implementation details, including parameters used, we refer the reader to the appendix.

    \subsubsection*{Shape Estimation while Stabilizing Cart Double Pendulum}
        Our first example demonstrates the functionality of our algorithm for estimating a sinusoidal shape while simultaneously
        balancing a cart double pendulum in its upright position.
        The purpose of this example is to show that our method can synthesize actions that ensures the cart double pendulum
        is maintained upright while actively collecting data for estimating the shape.
        This example also serves the purpose of illustrating that our method can safely automate choosing when to stabilize and
        when to explore for data using approximate linear models of the robot
        dynamics and stabilizing policies derived from the approximate models.

        \begin{figure*}[h!]
            \centering
            \begin{subfigure}{0.24\textwidth}
                \includegraphics[width=\linewidth]{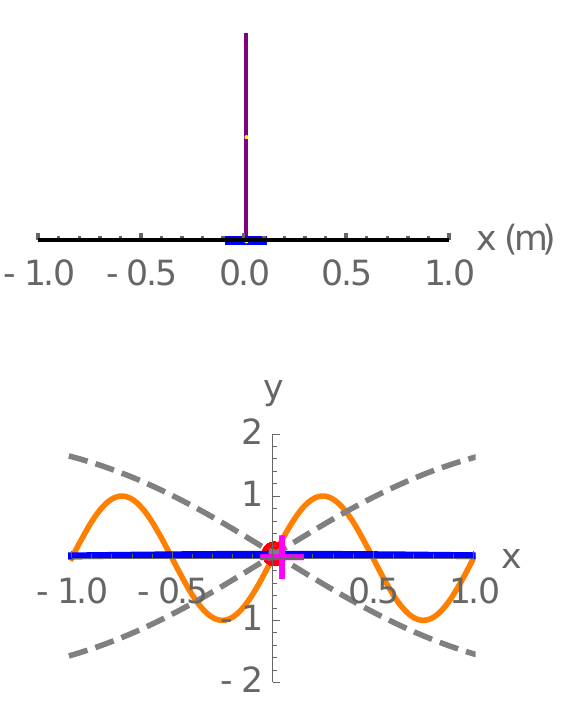}
                \caption{$t = 0$}
            \end{subfigure}
            \begin{subfigure}{0.24\textwidth}
                \includegraphics[width=\linewidth]{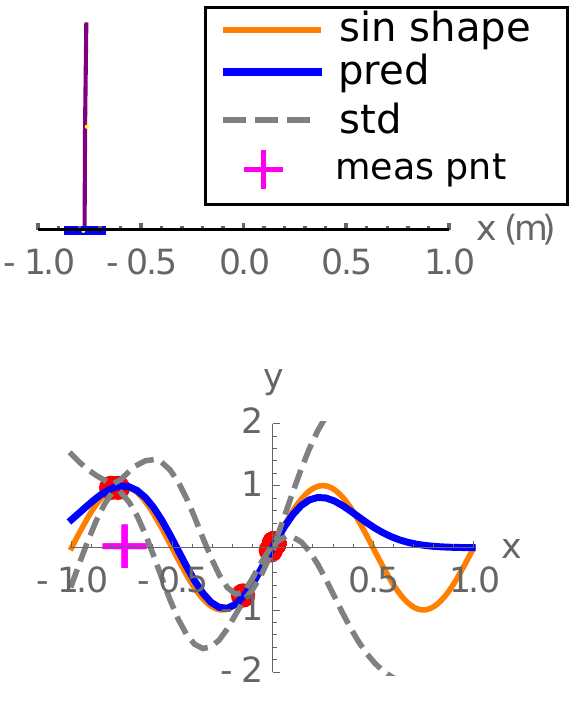}
                \caption{$t = 6$}
            \end{subfigure}
            \begin{subfigure}{0.24\textwidth}
                \includegraphics[width=\linewidth]{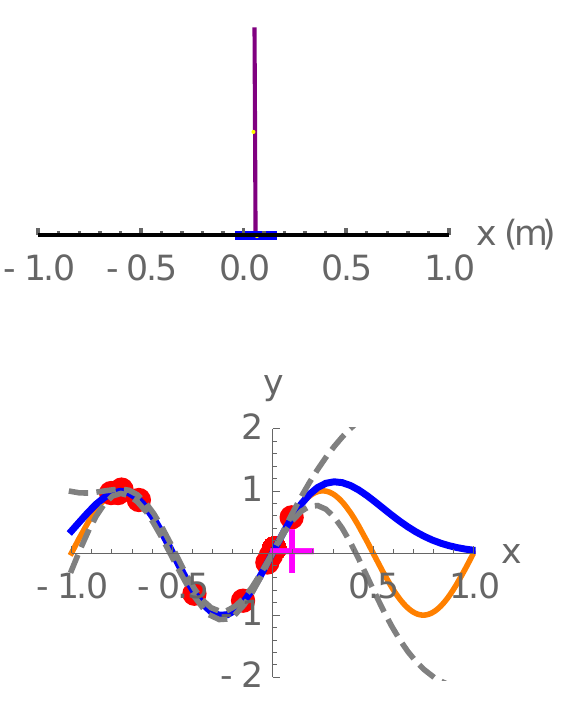}
                \caption{$t = 14$}
            \end{subfigure}
            \begin{subfigure}{0.24\textwidth}
                \includegraphics[width=\linewidth]{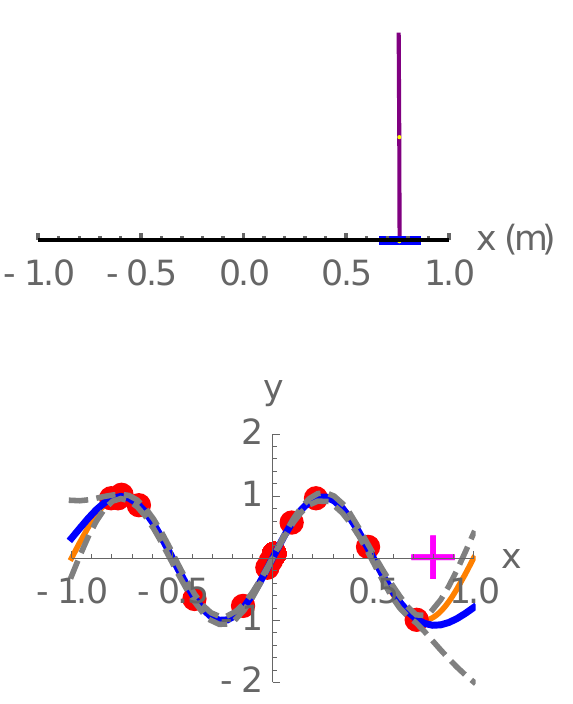}
                \caption{$t = 20$}
            \end{subfigure}
            \caption{
                    Time series snap-shots of cart double pendulum actively sampling and estimating the shape underneath.
                    The uncertainty (dashed gray lines) calculated from the collected data set drives the exploratory motion of
                    the cart double
                    pendulum while our method ensures that the cart double pendulum is maintained in its upright equilibrium
                    state.
            }
            \label{fig:cdp-time}
            \vspace{-5mm}
        \end{figure*}

        The measurements of the height of the sinusoidal shape are collected through $x$ position of the cart
        (illustrated in Fig.~\ref{fig:cdp-time} as the magenta crosshair underneath the cart).
        A Gaussian process with an radial basis function (RBF) kernel~\cite{nguyen_NEURO_incremental_sparse_gp} is then used to
        estimate the function and provide the distribution used for exploration.
        The underlying importance distribution (\ref{eq-importance-dist}) is updated as the data set is pruned to include
        new informative measurements.

        \begin{wrapfigure}{r}{0.4\textwidth}
            \vspace{0mm}
            \begin{subfigure}{0.4\textwidth}
                \includegraphics[width=\linewidth, trim={0 0 0 2mm}, clip]{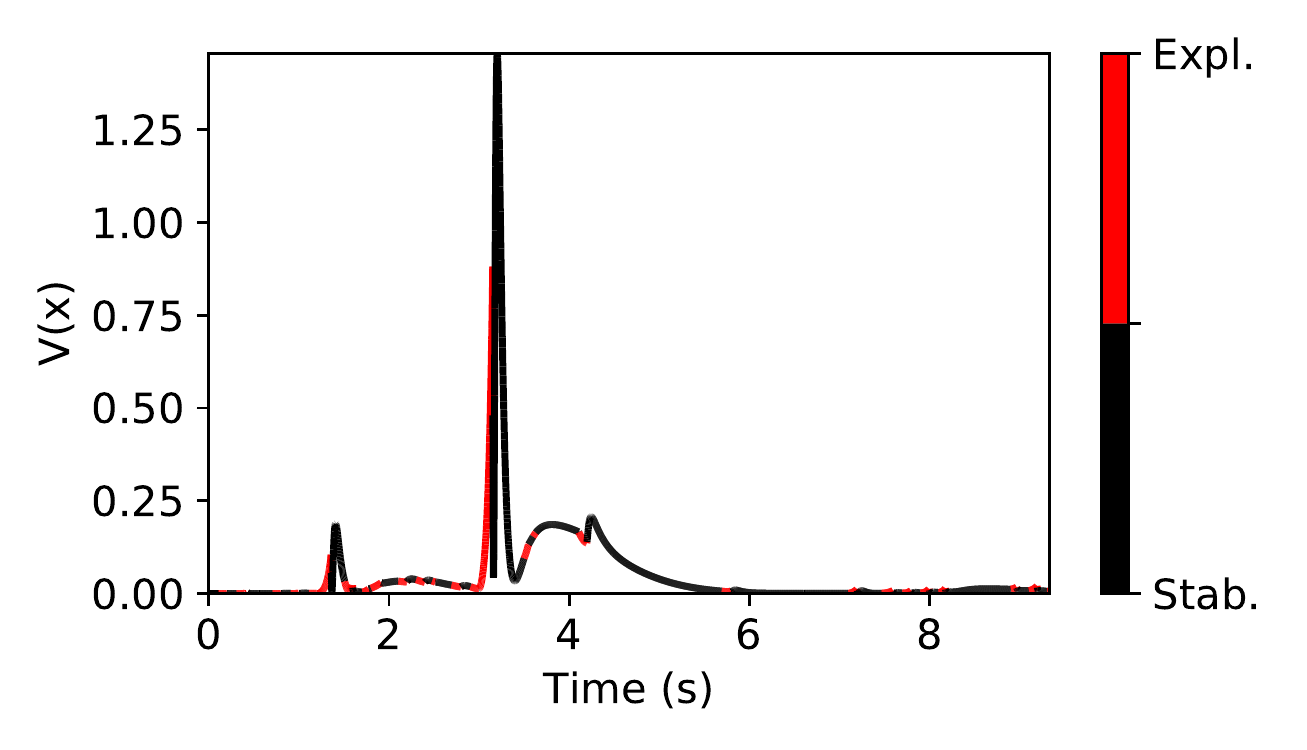}
            \end{subfigure}
            \begin{subfigure}{0.4\textwidth}
                \includegraphics[width=\linewidth]{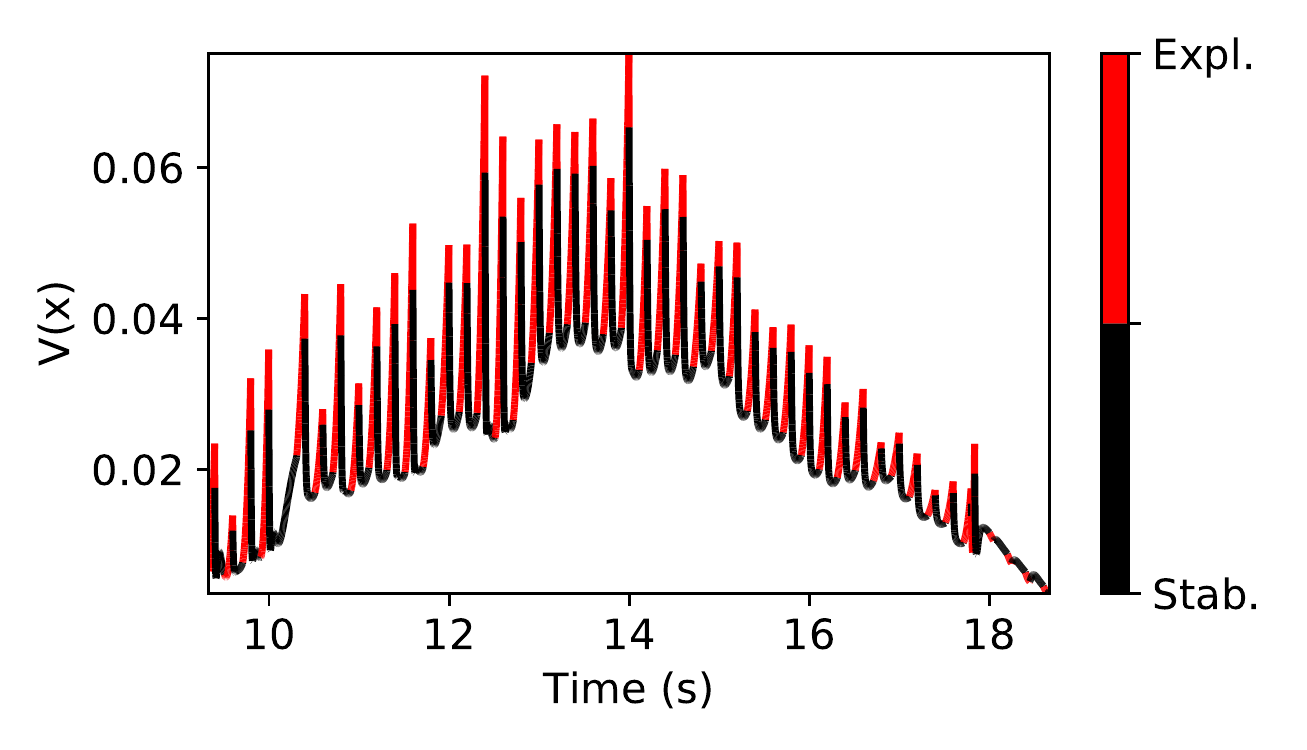}
            \end{subfigure}
            \begin{subfigure}{0.4\textwidth}
                \includegraphics[width=\linewidth]{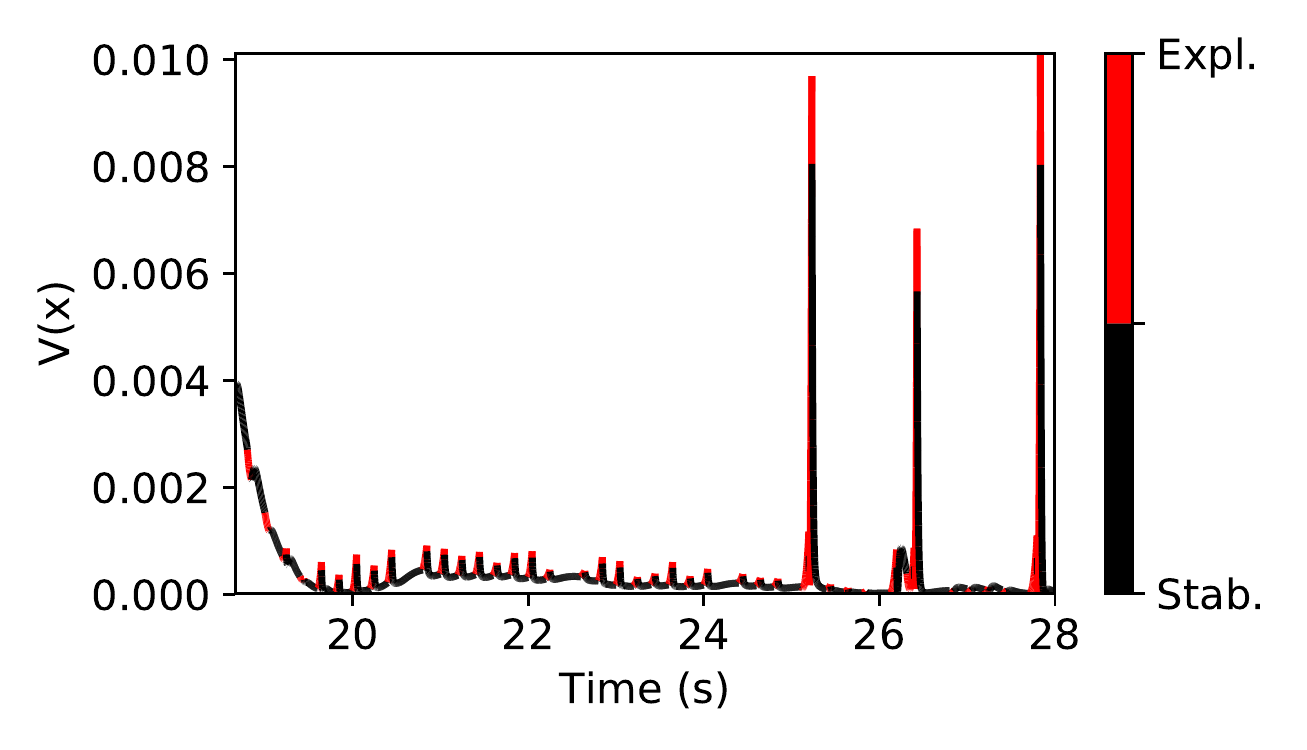}
            \end{subfigure}
            \caption{
                Lyapunov function for the cart double pendulum with upright equilibrium.
                The red line indicates when the active exploration control is applied.
                Lyapunov attractiveness property is illustrated through automatic switching of the exploration process.
            }
            \label{fig:lyap_cdp}
            \vspace{-15mm}
        \end{wrapfigure}

        As a result of Algorithm~\ref{alg-AOL}, the robot will spend time where there is
        a high probability of acquiring informative data.
        This results is the shape reconstruction shown in Fig.~\ref{fig:cdp-time} using a limited fixed set of data ($M=50$).

        We analyze our algorithm by observing a candidate Lyapunov function (energy).
        Figure~\ref{fig:lyap_cdp} depicts the value of the Lyapunov function over the time window of the cart double pendulum
        collecting data for estimating shape.
        The control vector $\mu_\star(t)$ over the application time $t \in \left[ \tau, \tau + \lambda \right]$ increases the
        overall energy in the system (due to exploration).
        Since we include a regularization term $R$ that ensures $\mu_\star$ does not deviate too far from the equilibrium
        policy $\mu(x)$, the cart double pendulum is able to stabilize itself, eventually returning to an equilibrium state and
        ensuring stability, illustrating the Lyapunov attractiveness property proven in Theorem~\ref{thm2}.

    \vspace{-2mm}
    \subsubsection*{Learning Dynamics of Quadrotor}
        Our next example illustrates active data acquisition in the state-space of a 22 degree of freedom quadrotor
        vehicle shown in Fig.~\ref{fig-env-1}.
        The results are averaged across $30$ trials with random initial conditions sampled uniformly in the body
        angular and linear velocities $\omega, v \sim \mathcal{U}\left[- 0.01, 0.01 \right]$ where $\mathcal{U}$ is a uniform distribution.

        The goal for this quadrotor is to maintain hovering height while collecting data in order to learn the
        dynamics model $f(x,u)$.
        In this example, a linear approximation of the dynamics centered at the hovering height is used as the local
        dynamics approximation on which Algorithm~\ref{alg-AOL} is based.
        We then generate a LQR controller with the approximate dynamics which we use as the equilibrium policy,
        The input data we collect is the state $x(t_i) = x_i$ and control $u(t_i) = u_i$ and the output data is
        $(x_{i+1}-x_{i})/(t_{i+1} -t_i)$ which approximates the function
        $\frac{\Delta x}{\Delta t} \approx \dot{x} = f(x,u)$~\cite{nagabandi2017neural}.
        An incremental sparse Gaussian process~\cite{nguyen_NEURO_incremental_sparse_gp} with a radial basis function
        kernel is used to generate a learned model of the dynamics using a data set of $M=80$ and to specify the
        importance measure (\ref{eq-importance-measure}).

                \begin{figure*}[th!]
                    \vspace{-5mm}
                    \centering
                    \includegraphics[width=\linewidth]{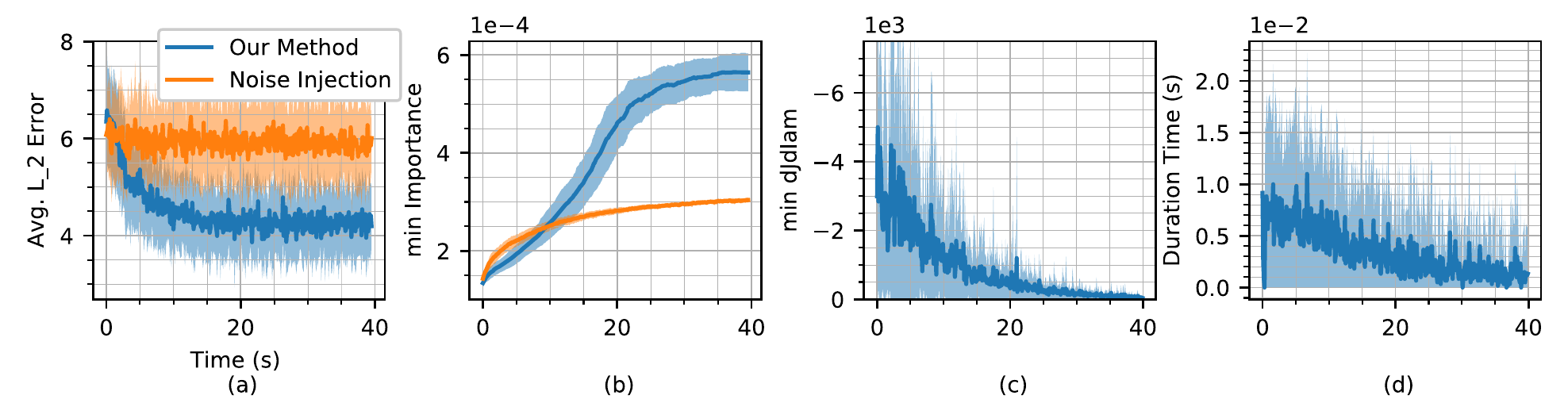}
                      \caption{
                          (a) Average $L_2$ error of the learned dynamics model evaluated on 10 uniformly distributed random
                          samples between $\left[-1,1\right]$ in the state and action space.
                          (b) Minimum importance measure from the collected data set using our method compared against injected noise.
                          (c) Minimum value of the mode insertion gradient evaluated at the chosen $\mu_\star(\tau)$ value.
                          (d) Calculated duration time of control $\mu_\star(\tau)$.  }
                     \label{fig:quad-dynamics-data}
                     \vspace{-5mm}
                \end{figure*}

        Figure~\ref{fig:quad-dynamics-data} (a) and Figure~\ref{fig:quad-dynamics-data} (b) illustrates the modeling error
        and the minimum importance value within the data set using our method and the equilibrium policy with uniformly
        added added noise at $33 \%$ of the saturation limit.
        Our method sequences and automates the process of choosing when it is best to explore and to stabilize by taking
        into account the approximate dynamics and the equilibrium policy.
        As a result, a robot is capable of acquiring informative data that improves the prediction of the nonlinear
        dynamic model of the quadrotor.
        In contrast, adding noise to the control input (often referred to as ``motor babble''~\cite{nagabandi2017neural})
        does not have temporal dependencies. That is, each new sample does not have information from the previous samples
        and cannot effectively explore the state-space.

        As the robot continues to explore, the value of the mode insertion gradient (\ref{eq:mode-insertion}) decreases as
        does the duration time $\lambda$ as shown in Fig.~\ref{fig:quad-dynamics-data} (c) and (d).
        This implies that the robot is sufficiently reducing the objective for area coverage and the equilibrium policy
         begins to take over to stabilize the robot.
        This is a result of taking into account the local stability of the robotic system while generating exploratory actions.

    \vspace{-5mm}
    \subsubsection*{Learning to Gallop}

        \begin{wrapfigure}{r}{0.6\textwidth}
            \vspace{-15mm}
            \centering
            \includegraphics[width=1\linewidth]{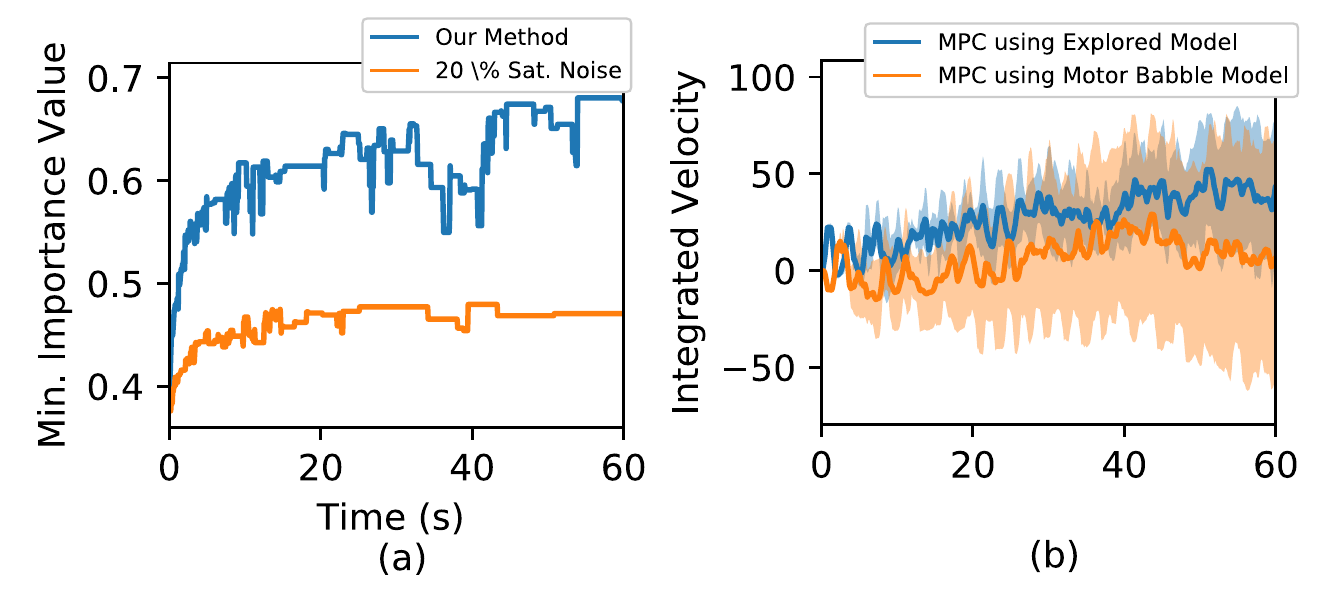}
            \caption{
                    \footnotesize
                    (a) Comparison of the minimum importance measure of the the data set for the half-cheetah example from a stable
                    standing policy with added $20 \%$ saturation limit noise and our approach for active data acquisition.
                    (b) Integrated forward velocity values from using the learned half-cheetah dynamics model in a model-predictive
                    control setting with standard deviation illustrated for 5 trials.
                    Our method is shown to collect data which provides a better dynamic model and a positive net forward velocity with
                    reduced variance.
                }
            \label{fig-cheetah-data}
            \vspace{-8mm}
        \end{wrapfigure}

        In this last example, we consider applications of Algorithm~\ref{alg-AOL} for systems with dynamic models and policies
        that are learned.
        We use the half-cheetah from the roboschool environment~\cite{klimov2017roboschool} for the task of learning a dynamics
        model in order to control the robot to gallop forward.

        We first learn a simple standing upright policy using the augmented random search (ARS) method~\cite{mania2018simple}.
        In that process, we collect the state and action data to compute a linear approximation using least-squares for
        the local dynamics.
        Then Algorithm~\ref{alg-AOL} is applied using an incremental sparse Gaussian process using an RBF kernel to generate a
        dynamics model from data as well as provide the importance measure using a set of $M=40$ data points.
        The input-output data structure maps input $(x(t_i), u(t_i))$ to the change in state $\frac{\Delta x}{\Delta t}$.
        Our running cost $\ell(x,u)$ is set to maintain the half-cheetah upright.
        After the Gaussian process model is learned, we use the generated model in the prediction of the forward dynamics
        as a replacement for the initial dynamics model.

        As shown in Fig.~\ref{fig-cheetah-data}, our method collects informative data while respecting the standing upright
        policy when compared to noisy inputs.
        We compared the two learned models using our controller with $D_\text{KL}=0$ and the running cost $\ell(x,u)$ set to
        maximize the forward velocity of the half-cheetah.
        We show those results in Fig.~\ref{fig-cheetah-data} over 5 runs of our algorithm at different initial states.
        Our method provides a learned model that has overall positive integrated velocity (forward movement).
        While our method is more complex than simply adding noise, it provides stability guarantees based on known policies
        in order to explore and collect data.

        \vspace{-5mm}

\section{Conclusion} \label{sec-conclusion}
    \vspace{-2mm}

    Algorithm~\ref{alg-AOL} enables robots to actively seek out informative data based on the learning task while maintaining
    stability using equilibrium policies.
    Our method generates area coverage using a KL-divergence measure in order to enable robots to actively seek out
    informative data.
    Moreover, by using a hybrid systems theory approach to generating area coverage, we were able to incorporate equilibrium
    policies in order to provide stability guarantees even with the model of the robot dynamics only locally known.
    Last, we provide examples that illustrate the benefits of our approach for active data acquisition for learning tasks.


\vspace{-4mm}
{\smaller
\bibliography{references}}

\end{document}